\documentclass{article}

\usepackage{mdframed}

\usepackage[preprint]{neurips_2023_otml}

\usepackage[utf8]{inputenc} \usepackage[T1]{fontenc}    \usepackage{hyperref}       \usepackage{url}            \usepackage{amsfonts}       \usepackage{nicefrac}       \usepackage{xcolor}         \usepackage{microtype}
\usepackage{graphicx}
\usepackage{subfigure}
\usepackage{booktabs} \usepackage{caption}

\usepackage{algorithm}
\usepackage{algorithmic}

\usepackage{amsmath}
\usepackage{amssymb}
\usepackage{mathtools}
\usepackage{amsthm}
\usepackage{thm-restate}

\usepackage{wrapfig}
\usepackage{nicematrix}

\usepackage[capitalize,noabbrev]{cleveref}

\newcommand{\R}{\mathbb{R}}

\newcommand{\C}{\mathbf{C}}

\newcommand{\K}{\mathbf{K}}

\newcommand{\Pb}{\mathbf{P}}
\newcommand{\Tb}{\mathbf{T}}
\newcommand{\Qb}{\mathbf{Q}}

\newcommand{\ind}{1\!\!1}

\newcommand{\eg}{\textit{e.g.}\ }
\newcommand{\ie}{\textit{i.e.}\ }

\newcommand{\KL}{\operatorname{KL}}
\newcommand{\diag}{\operatorname{diag}}

\newcommand{\integ}[1]{{[\![#1]\!]}}

\usepackage{bm}
\usepackage{bbm}
\usepackage{scalerel}
\usepackage{wrapfig}

\usepackage{wrapfig}

\DeclareMathOperator*{\argmax}{arg\,max}
\DeclareMathOperator*{\argmin}{arg\,min}

\definecolor{cadmiumred}{rgb}{0.89, 0.0, 0.13}

\definecolor{cadmiumgreen}{rgb}{0.0, 0.42, 0.24}

\definecolor{OliveGreen}{rgb}{0.33, 0.62, 0.18}

\theoremstyle{plain}
\newtheorem{theorem}{Theorem}

\newtheorem{assumption}[theorem]{Assumption}
\newtheorem{proposition}[theorem]{Proposition}

\theoremstyle{definition}

\newmdtheoremenv{prop}{Proposition}

\usepackage{graphicx, color}
\graphicspath{{fig/}}

\title{Optimal Transport with Adaptive Regularisation}

\author{  Hugues Van Assel$^{1 *}$, Titouan Vayer$^{2}$, Rémi Flamary$^{3}$, Nicolas Courty$^{4}$ \\
  $^{1}$UMPA ENS Lyon , $^{2}$LIP ENS Lyon, $^{3}$CMAP École polytechnique, \\
  $^{4}$IRISA Université Bretagne Sud  \\
  \\
    }

\begin{document}

\maketitle

\def\thefootnote{*}\footnotetext{\texttt{hugues.van\textunderscore assel@ens-lyon.fr}}\def\thefootnote{\arabic{footnote}}

\begin{abstract}
  Regularising the primal formulation of optimal transport (OT) with a strictly convex term leads to enhanced numerical complexity and a denser transport plan.
  Many formulations impose a global constraint on the transport plan, for instance by relying on entropic regularisation. 
  As it is more expensive to diffuse mass for outlier points compared to central ones, this typically results in a significant imbalance in the way mass is spread across the points.
  This can be detrimental for some applications where a minimum of smoothing is required per point.  
  To remedy this, we introduce OT with Adaptive RegularIsation (OTARI), a new formulation of OT that imposes constraints on the mass going in or/and out of each point.
  We then showcase the benefits of this approach for domain adaptation.
\end{abstract}

\section{Introduction}

Optimal transport (OT) is a well-established framework to compare probability distributions with numerous applications in machine learning \cite{arjovsky2017wasserstein, ozair2019wasserstein, peyre2019computational}.
Discrete OT seeks a transportation plan that minimizes the total transportation cost between samples from the source and target distributions.
In the absence of regularisation, this optimal OT plan is inherently sparse. 
Regularising OT with a strictly convex term is a widely adopted practical approach, leading to reduced numerical complexity and more diffuse OT plans \cite{peyre2019computational}.
As an illustration, the prominent entropic regularisation \cite{cuturi2013sinkhorn} leads to a dense plan.
In some applications, the smoothing effect induced by the regularisation has a primary importance on its own. A key example is the construction of doubly stochastic affinity matrices for clustering and dimensionality reduction \cite{landa2021doubly,Zass}, where smoothing enables connecting to neighbor points.
Another is domain adaptation \cite{courty2017joint} where smoothed OT often results in enhanced performance when compared to non-regularised ones (see for instance \cref{tab:da_exps}). Many OT regularisation schemes on the primal formulation impose a constraint on the overall transport plan.
Consequently, central data points tend to exhibit a denser (or more diffuse) transport plan compared to extreme (or outlier) data points, for which diffusion is more costly. As a result, the latter points receive limited benefits from the smoothing effect introduced by the regularizer as shown in the left side of \cref{fig:entropic_ot_plans}. This partly explains OT's significant sensitivity to outliers in many applications \cite{mukherjee2021outlier, pmlr-v202-chuang23a}. 
To remedy this, one needs to constrain the transport plan in a pointwise manner.
Note that this has recently been explored for constructing affinity matrices \cite{van2023snekhorn} (\ie symmetric OT setting) leading to enhanced noise robustness and clustering abilities.

\textbf{Contributions.} In this work, we develop a new formulation of OT, called OT with Adaptive RegularIsation (OTARI), allowing to control, for any strictly convex function $\psi$, the value of $\psi$ on each row and/or column of the OT plan. 
We then show the advantages of OTARI over usual regularised OT on domain adaptation tasks, focusing particularly on the negative entropy and the $\ell_2$ norm respectively associated with entropic \cite{cuturi2013sinkhorn} and quadratic \cite{blondel2018smooth} optimal transport.

\section{Regularised Optimal Transport}

We first introduce the discrete OT problem before presenting regularised formulations and associated algorithms.
Let $X_S = \{ \bm{x}_i^S \in \R^d \}_{i=1}^{N_S}$ and $X_T = \{ \bm{x}_i^T \in \R^d \}_{i=1}^{N_T}$ denote the sets of respectively source and target point locations. The discrete \emph{Monge-Kantorovitch} problem \cite{kantorovich1942transfer} focuses on the optimal allocation strategy to transport the empirical measure $\mu_S = \frac{1}{N_S} \sum_{i=1}^{N_S} a_i  \delta_{\bm{x}_i^S}$ onto $\mu_T = \frac{1}{N_T} \sum_{i=1}^{N_T} b_i  \delta_{\bm{x}_i^T}$ where $\bm{a} \in \Delta^{N_S}$ and $\bm{b} \in \Delta^{N_T}$. It consists in computing a \emph{coupling} $\Pb \in \R_{+}^{N_S \times N_T}$ \ie a joint probability measure over $X_S \times X_T$ solving the linear program
\begin{equation}
    \tag{OT}
    \label{eq:OT}
    \min_{\Pb \in \Pi(\bm{a}, \bm{b})} \: \langle \Pb, \C \rangle
\end{equation}
where $\Pi(\bm{a}, \bm{b})$ is the transport polytope with marginals $(\bm{a}, \bm{b})$ and the \emph{cost matrix} $\mathbf{C} \in \R_+^{N_S \times N_T}$ encodes the pairwise distances between the source and target samples. One can typically consider the squared Euclidean distance $C_{ij} = \|\bm{x}^S_{i}-\bm{x}^T_{j}\|_2^2$ or any Riemannian distance over a manifold \cite{villani2009optimal}.

To enable faster algorithmic resolution as well as smoother solutions, one can rely on a strictly convex regulariser $\psi : \R^{N_s} \to \R$. It amounts to solving $\min_{\Pb \in \Pi(\bm{a}, \bm{b})} \: \langle \Pb, \C \rangle + \varepsilon^\star \sum_i \psi(\Pb_{i:})$ where $\varepsilon^\star > 0$. 
Interestingly, regularised OT can also be framed using a convex constraint as follows
\begin{align}\label{eq:cot}
    \min_{\Pb \in \Pi(\bm{a}, \bm{b})} \: \langle \Pb, \C \rangle \quad \text{s.t.} \quad  \Pb \in \overline{\mathcal{B}}(\eta)
    \tag{ROT}
\end{align}
where $\overline{\mathcal{B}}(\eta) \coloneqq \{ \Pb \: \text{s.t.} \: \sum_i \psi(\Pb_{i:}) \leq \eta \}$. Note that the previously introduced $\varepsilon^\star$ is the optimal dual variable associated with the constraint $\overline{\mathcal{B}}(\eta)$ in the above equivalent formulation. 
Throughout, we make the following assumption on $\psi$.
\begin{assumption}\label{assumption_psi}
    Let $\psi: \mathrm{dom}(\psi) \to \R \cup \{\infty\}$ be strictly convex and differentiable on the interior of its domain $\mathrm{dom}(\psi) \subset \R^{N_S}_+$.
\end{assumption}

In what follows, we denote by $\psi(\Pb) = (\psi(\Pb_{1:}), ..., \psi(\Pb_{N_S:}))^\top$.
We introduce $\psi^*:= \mathbf{p} \to \sup_{\mathbf{q} \in \mathrm{dom}(\psi)} \langle \mathbf{p}, \mathbf{q} \rangle - \psi(\mathbf{q})$ the convex conjugate of $\psi$ \cite{rockafellar1997convex}.
Note that when $\psi$ is strictly convex, this supremum is uniquely achieved and from Danskin's theorem \cite{danskin1966theory}: $\nabla \psi^*(\bm{p}) = \argmax_{\mathbf{q} \in \mathrm{dom}(\psi)} \langle \mathbf{p}, \mathbf{q} \rangle - \psi(\mathbf{q})$. 
We show in \cref{sec:proof_global} that when $\varepsilon^\star > 0$, \textit{i.e.}\ when the constraint $\Pb \in \overline{\mathcal{B}}(\eta)$ is active, \eqref{eq:cot} is solved for $\Pb^\star = \nabla \psi^*((\C - \bm{\lambda}^\star \oplus \bm{\mu}^\star) / \varepsilon^\star)$ \footnote{We use the notation $\nabla \psi^*(\Pb) \coloneqq (\nabla \psi^*(\Pb_{1:}), ..., \nabla \psi^*(\Pb_{N_S:}))^\top$.} where $(\bm{\lambda}^\star, \bm{\mu}^\star, \varepsilon^\star)$ solve the following dual problem 
\begin{align}
    \max_{\bm{\lambda}, \bm{\mu}, \varepsilon>0}  \: \langle \bm{\lambda}, \bm{a} \rangle + \langle \bm{\mu}, \bm{b} \rangle + \varepsilon \Big(\sum_i \psi^*((\C_{i:} - \lambda_i \bm{1} - \bm{\mu}) / \varepsilon) - \eta \Big) \:.
    \tag{Dual-ROT}
\end{align}
The above objective is concave thus the problem can be solved exactly using \textit{e.g.}\ BFGS \cite{liu1989limited} or ADAM \cite{kingma2014adam}.
As a complementary view, one can also frame \eqref{eq:cot} as a $\psi$-Bregman projection over a convex set. The $\psi$-Bregman divergence is defined as $D_\psi(\Pb, \Qb) := \psi(\Pb) - \psi(\Qb) - \langle \Pb - \Qb, \nabla \psi(\Qb) \rangle$. The solution of \eqref{eq:cot} can then be expressed as
$\Pb^\star = \operatorname{Proj}^{D_\psi}_{\Pi(\bm{a}, \bm{b}) \cap \overline{\mathcal{B}}(\eta)}(\K_\sigma)$
where $\K_\sigma \coloneqq \nabla \psi^*(- \C / \sigma)$ for any $\sigma < \varepsilon^\star$ (see \cref{sec:proof_global} for details). The key benefit of the above result is that it enables solving \eqref{eq:cot} with alternating Bregman projection schemes \cite{benamou2015iterative}.

In this work, we focus specifically on 
certain Bregman divergences: the Kullback Leibler ($\KL$) divergence and the squared Euclidean distance.
The first reads $D_{\KL}(\Pb | \Qb) = \langle \Pb, \log \left(\Pb \oslash \Qb \right) - \bm{1} \bm{1}^\top \rangle$ with associated negative entropy $\psi_{\KL}(\mathbf{p}) = \langle \mathbf{p}, \log \mathbf{p} - 1 \rangle$. In this case, \eqref{eq:cot} boils down to entropic OT and solved for $\operatorname{Proj}^{\KL}_{\Pi(\bm{a}, \bm{b})}(\K_{\varepsilon^\star})$ where $\K_{\varepsilon^\star} = \nabla \psi_{\KL}^{*}(-\C / \varepsilon^\star) = \exp(-\C / \varepsilon^\star)$ is a Gibbs kernel. This projection is well-known as the \emph{static Schrödinger bridge} \cite{leonard2013survey} referring to statistical physics where it first appeared \cite{schrodinger1931umkehrung}, and can be computed efficiently using the Sinkhorn algorithm \cite{cuturi2013sinkhorn}. For the squared Euclidean distance, we define $\psi_{2}(\mathbf{p}) = \frac{1}{2} \| \mathbf{p} \|^2_2$. The associated problem \eqref{prop:cot} is usually referred to as quadratic OT \cite{lorenz2021quadratically} and can yield sparse OT plans unlike entropic.

\section{Optimal Transport with Adaptive Regularisation}\label{sec:pwot}

In this section, we present a new formulation of OT that imposes constraints on each row of the OT plan.
We begin by introducing the set of matrices with \emph{point-wise} constraints. To set the upper bound, we rely on the perplexity parameter $\xi$ \cite{van2023snekhorn} that can be interpreted as the number of effective neighbors for each point. Concretely, we define $\mathbf{e}_{\xi} = \frac{1}{\xi}(\ind_{i \leq \xi})_{i}$ and
\begin{align}
  \mathcal{B}_\psi(\xi) &\coloneqq \{\Pb \bm{\geq} \bm{0} \: \text{s.t.}\ \: \forall i, \: \psi(\Pb_{i:}) \leq \psi(\mathbf{e}_{\xi}) \} \:.
\end{align}
Note that $\psi_{\KL}(\mathbf{p}_{\xi}) = - (\log \xi + 1)$ and $\psi_2(\mathbf{p}_{\xi}) = 1/\xi$.
We now define Optimal Transport with Adaptive RegularIsation (OTARI) as the generalization of \eqref{eq:cot} to the case where the strictly convex constraint is given by $\mathcal{B}_\psi(\xi)$. Similarly to \cref{prop:cot} (\cref{sec:proof_global}), we can frame OTARI as a $\psi$-Bregman projection of $\K_\sigma = \nabla \psi^*(-\C / \sigma)$ or solve it using dual ascent.
\begin{proposition}\label{prop:pcot}
  Let $(\bm{a}, \bm{b}, \xi)$ be such that $\Pi(\bm{a}, \bm{b}) \cap \mathcal{B}_\psi(\xi)$ has an interior point and let $\Pb^\star$ solve
  \begin{align}\label{eq:pcOT}
    \min_{\Pb \in \Pi(\bm{a}, \bm{b})} \: \langle \Pb, \C \rangle \quad \text{s.t.} \quad  \Pb \in \mathcal{B}_\psi(\xi) \:.
    \tag{OTARI-s}
\end{align}
Let $\bm{\varepsilon}^\star$ be the optimal dual variable associated with the constraint $\Pb \in \mathcal{B}_\psi(\xi)$.
If $\bm{\varepsilon}^\star \bm{>} \bm{0}$, then it holds $\Pb^\star = \operatorname{Proj}^{D_\psi}_{\Pi(\bm{a}, \bm{b}) \cap \mathcal{B}_\psi(\xi)}(\K_\sigma)$ for any $0 < \sigma \leq \min_i{\varepsilon_i^\star}$. Moreover it holds $\Pb^\star = \nabla \psi^* \left(\diag(\bm{\varepsilon}^\star)^{-1} (\C - \bm{\lambda}^\star \oplus \bm{\mu}^\star) \right)$ where $(\bm{\lambda}^\star, \bm{\mu}^\star, \bm{\varepsilon}^\star)$ solve the following dual
\begin{align}
  \max_{\bm{\lambda}, \bm{\mu}, \bm{\varepsilon} \bm{>} \bm{0}} \: \langle \bm{\lambda}, \bm{a} \rangle + \langle \bm{\mu}, \bm{b} \rangle + \left\langle \bm{\varepsilon}, \psi^*\left(\diag(\bm{\varepsilon})^{-1} (\C - \bm{\lambda} \oplus \bm{\mu}) \right) - \psi(\mathbf{e}_\xi) \bm{1}  \right\rangle \:.
  \tag{Dual-OTARI-s}
\end{align}
\end{proposition}

\begin{figure*}[t]
  \begin{center}
  \centerline{\includegraphics[width=\columnwidth]{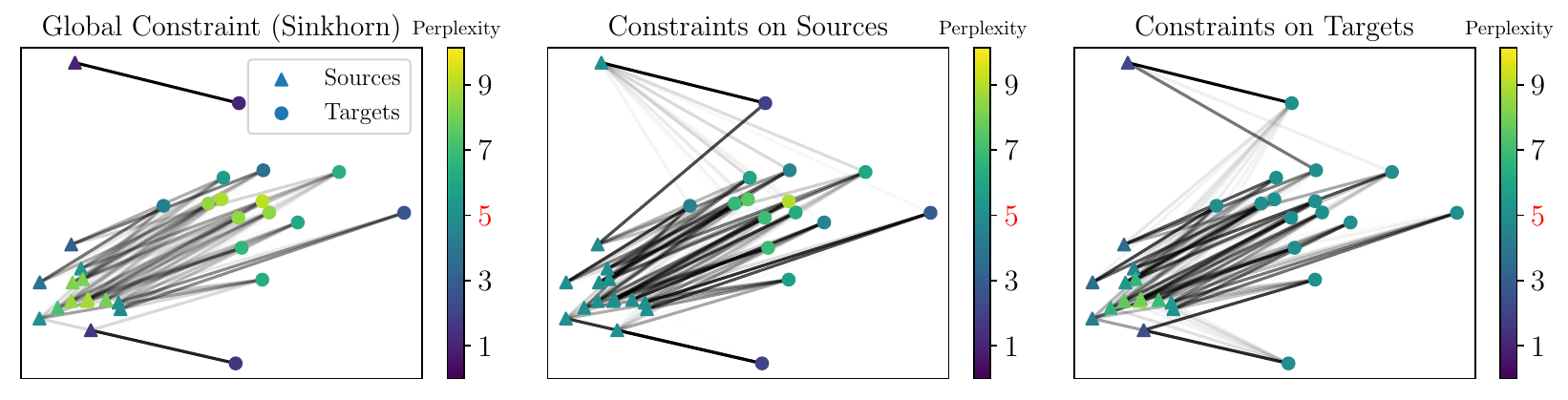}}
  \caption{Entropic OT plans ($\xi=5$) with global constraint, pointwise constraints on sources and then on targets. The three plans have the same global entropy. The color of each source (resp. target) point shows the perplexity (exponential of entropy) of the associated row (resp. column) of the OT plan.}
  \label{fig:entropic_ot_plans}
  \end{center}
  \vspace{-0.5cm}
\end{figure*}

\begin{wrapfigure}[15]{L}{0.52\textwidth}
  \begin{minipage}{0.52\textwidth}
\begin{algorithm}[H]
  \caption{\textit{Dykstra} for solving (OTARI-d)}
  \label{algo:Dykstra_pcot}
  \begin{algorithmic}[1]
      \STATE {\textbf{Input}: $\C$, $\psi(\cdot)$, $\xi^{a}$, $\xi^{b}$, $\varepsilon$}, $\bm{a}$, $\bm{b}$ \\
            \STATE $\left(\Pb_b, \bm{\Xi}, \bm{\Theta} \right) \leftarrow \left(\nabla \psi^{\star}(-\C / \varepsilon), \mathbf{0}, \mathbf{0} \right)$ \\
      \WHILE{not converged}
                    \STATE $\Pb_a \leftarrow \operatorname{Proj}^{D_\psi}_{\Pi(\bm{a})}(\Pb_b)$ 
          \\
          \STATE $\overline{\Pb}_{a} \leftarrow \operatorname{Proj}^{D_\psi}_{\mathcal{B}_{\psi}(\xi^{a})}\circ \nabla \psi^*(\nabla \psi(\Pb_a) + \bm{\Xi})$ 
          \\
          \STATE $\bm{\Xi} \leftarrow \bm{\Xi} + \nabla \psi(\Pb_{a}) - \nabla \psi(\overline{\Pb}_{a})$
          \\
          \STATE $\Pb^\top_b \leftarrow \operatorname{Proj}^{D_\psi}_{\Pi(\bm{b})}(\overline{\Pb}_{a}^\top)$ 
          \\
          \STATE $\overline{\Pb}_{b}^\top \leftarrow \operatorname{Proj}^{D_\psi}_{\mathcal{B}_{\psi}(\xi^{b})}\circ \nabla \psi^*((\nabla \psi(\Pb_b) + \bm{\Theta})^\top)$ 
          \\
          \STATE $\bm{\Theta} \leftarrow \bm{\Theta} + \nabla \psi(\Pb_{b}) - \nabla \psi(\overline{\Pb}_{b})$
      \ENDWHILE  
      \STATE {\bfseries Output: $\overline{\Pb}_{b}$}
\end{algorithmic}
\end{algorithm}
\end{minipage}
\end{wrapfigure}

According to \cref{prop:pcot}, one can solve \eqref{eq:pcOT} using either alternating projections or dual ascent. 
When $\bm{\varepsilon}^\star \bm{>} \bm{0}$, meaning that all constraints are active \ie $\forall i, \: \psi(\Pb^\star_{i:}) = \psi(\mathbf{p}_{\xi})$, dual ascent is usually faster. However, if $\bm{\varepsilon}^\star$ has null components, one can still rely on $\operatorname{Proj}^{D_\psi}_{\Pi(\bm{a}, \bm{b}) \cap \mathcal{B}_\psi(\xi)}(\K_\varepsilon)$ to provide an approximate solution as alternating Bregman projections are always guaranteed to converge.

Note that we can impose the pointwise constraint equivalently on the rows or the columns of the OT plan. Hence (OTARI-t) can be defined by imposing the constraint on the target samples \textit{i.e.}\ $\Pb^\top \in \mathcal{B}_\psi(\xi)$.
We also propose a doubly constrained formulation called (OTARI-d) that consists of projecting $\K_\sigma$ onto the nonempty set $\mathcal{B}_\psi(\xi^{a}) \cap \mathcal{B}^\top_\psi(\xi^b)$ 
where we defined $\mathcal{B}^\top_\psi(\xi) = \{ \Pb^\top \in \mathcal{B}_\psi(\xi)\}$ thus ensuring sufficient smoothing for both rows and columns.
Such projection can be computed using alternating Bregman projections, whose convergence has been well-studied \cite{censor1998dykstra, benamou2015iterative}.
As we generally do not have access to a closed form for the projection onto the transport polytope $\Pi(\bm{a}, \bm{b})$, it is common to alternate projection onto $\Pi(\bm{a})$ and $\Pi(\bm{b})$ separately (see \eg the seminal Sinkhorn algorithm \cite{cuturi2013sinkhorn}).
We extend this scheme by adding projection steps into the pointwise constraints $\mathcal{B}_\psi(\xi)$ for both $\Pb$ and $\Pb^\top$. As this set is not affine, one needs to resort to the Dykstra procedure \cite{dykstra1983algorithm} that can be applied to a broad class of Bregman divergences \cite{bauschke2000dykstras}, as shown in \cref{algo:Dykstra_pcot}.
In \cref{sec:proof_projs}, we provide the form of the projections for $\psi_{\KL}$ and $\psi_2$. A key benefit of decoupling both row and column constraints is that projection onto $\mathcal{B}_{\psi}(\xi)$ exhibits a simple structure where the rows can be decoupled into independent subproblems.

\section{Application to Domain Adaptation}

\begin{figure}[h]
  \begin{center}
  \centerline{\includegraphics[width=\columnwidth]{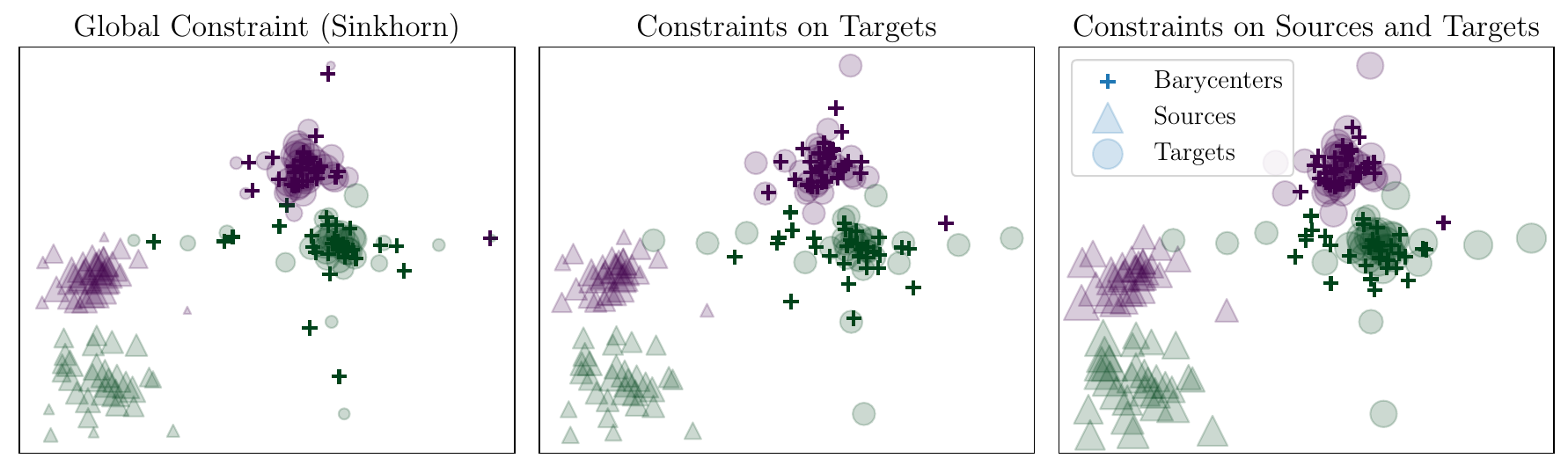}}
  \caption{Toy domain adaptation scenario with entropic OT plans ($\xi=10$) with various constraints. The size of the point is proportional to the associated entropy. When using Sinkhorn, barycentric mapping match outliers points since the OT plan is less diffuse for these points. In turn, using pointwise constraints concentrate the 
   mapped points in high-density regions, thus giving more robust estimates for the mappings onto the target domain.}
  \label{fig:visu_DA}
  \end{center}
  \vspace{0.2cm}
\end{figure}

\begin{table}[h]
          \centering
        \begin{small}
    \begin{tabular}{lc@{\hskip 0.1in}c@{\hskip 0.1in}c@{\hskip 0.1in}c@{\hskip 0.1in}c@{\hskip 0.1in}c}
    \toprule[1.5pt]
    & OT& EOT & EOTARI-s & EOTARI-t & EOTARI-d \\
    \midrule
    MNIST $\to$ USPS ($\xi=30$) & $53.1(5.4)$ & $64.2(2.8)$ & $65.0(5.3)$ & $66.4(3.5)$ & $\mathbf{67.4(2.9)}$ \\
    MNIST $\to$ USPS ($\xi=300$) & $53.1(5.4)$ & $68.8(3.1)$ & $70.8(4.2)$ & $70.2(3.4)$ & $\mathbf{72.6(5.1)}$ \\
  
    USPS $\to$ MNIST ($\xi=30$) & $59.1(4.9)$ & $60.8(5.4)$ & $61.6(4.4)$ & $\mathbf{62.6(3.0)}$ & $61.0(4.7)$ \\
    USPS $\to$ MNIST ($\xi=300$) & $59.1(4.9)$ & $59.8(1.6)$ & $61.0(2.3)$ & $\mathbf{61.6(3.0)}$ & $58.8(2.3)$ \\
    \midrule
    \midrule
    & OT & QOT & QOTARI-s & QOTARI-t & QOTARI-d \\
    \midrule
    MNIST $\to$ USPS ($\xi=30$) & $53.1(5.4)$ & $68.3(3.9)$ & $68.3(3.6)$ & $\mathbf{69.3(4.7)}$ & $68.1(4.6)$ \\
    MNIST $\to$ USPS ($\xi=300$) & $53.1(5.4)$ & $60.7(1.5)$ & $\mathbf{67.0(2.4)}$ & $65.5(2.3)$ & $65.8(2.5)$ \\
    USPS $\to$ MNIST ($\xi=30$) & $59.1(4.9)$ & $60.4(3.5)$ & $\mathbf{62.8(3.7)}$ & $59.6(2.7)$ & $61.6(3.1)$ \\
    USPS $\to$ MNIST ($\xi=300$) & $59.1(4.9)$ & $59.2(3.4)$ & $60.1(3.0)$ & $\mathbf{62.0(3.7)}$ & $61.5(3.8)$ \\
    \bottomrule[1.5pt]
    \end{tabular}
    \end{small}
    \vspace{0.4cm}
    \caption{Domain adaptation 1-NN classification scores for OT (unregularised), EOT (entropic), EOTARI (entropic OTARI), QOT (quadratic), QOTARI (quadratic OTARI) for $\xi=30$ and $\xi=300$.}
    \label{tab:da_exps}
    \vspace{0.1cm}
    \end{table}

In this section, we evaluate OTARI on a domain adaptation task where the goal is to transport labeled data points to a target domain where a classifier is trained. Mapping onto the target domain is performed through a barycentric mapping of the form: for any $i \in \integ{N_S}$, $\hat{\bm{x}}_i = \frac{1}{a_i} \sum_j T_{ij} \bm{x}_j^T$.
Looking at \cref{fig:visu_DA}, one can notice that using OTARI for domain adaptation yields a mapping that is concentrated in high-density (thus more faithful) regions of the target domain. On the opposite, when using globally constrained OT (left side of \cref{fig:visu_DA}), the barycentric mapping associated with an outlier is concentrated on the outlier's position. For the experiments, we take $\C$ as the squared Euclidean distance computed from raw images of the handwritten digit classification benchmark MNIST-USPS. 
Following the standard practice in OT-based domain adaptation \cite{flamary2016optimal}, we map the source to the target samples and then train a 1-NN classifier on the barycentric mappings with source labels.
We compute the outcomes across 10 independent trials. In each of these experiments, the target data is partitioned into a 90\% training and 10\% testing split, with OT barycentric mappings and 1-NN classifiers exclusively applied to the training set. Mean scores and standard deviations are displayed in \cref{tab:da_exps}. The latter shows that adaptive regularisation consistently outperforms global regularisation (set such that the average perplexity is $\xi$ for a fair comparison) with significant performance gains in some settings (see \textit{e.g.}\ MNIST $\to$ USPS ($\xi=300$) with the quadratic regularisation).
\section{Conclusion}

In this work, we presented a versatile framework to control the value of any OT regulariser in source or/and target locations. We showed encouraging preliminary results for domain adaptation that will be investigated in upcoming works. One could also extend OTARI to continuous distributions and apply it to OT mapping
estimation \cite{pooladian2021entropic}. Other interesting directions include
investigating optimization algorithms that can avoid quadratic memory complexity.

\newpage
\appendix

\section{Notations}\label{sec:notations}

We adopt the conventions that $0/0 = 0$, $0 \log(0) = 0$ and $x/0 = \infty$ for $x > 0$.
$\exp$, $\log$ applied to vectors/matrices are taken element-wise.  
$\bm{1}$ is the all-one vector whose size depends on the context.
$\langle \cdot, \cdot \rangle$ is the standard inner product for matrices/vectors. 
$P_{ij}$ denotes the entry at position $(i,j)$ of a matrix $\mathbf{P}$ while $\mathbf{P}_{i:}$ and denotes the $i$-th row. 
$\Pb \bm{\geq} \bm{0}$ means that for any $(i,j)$, $P_{ij} \geq 0$.
$\odot$ (\textit{resp.} $\oslash$) stands for element-wise multiplication (\textit{resp.} division) between vectors/matrices. 
For $\bm{a}, \bm{b} \in \R^n, \bm{a} \oplus \bm{b} \in \R^{n \times n}$ is $(a_i + b_j)_{ij}$.
For $\alpha \in \R$, $\bm{p}^{\odot \alpha}$ and $\Pb^{\odot \alpha}$ denote element-wise exponentiation \ie $[\bm{p}^{\odot \alpha}]_i = p_i^\alpha$.
$[\Pb]_+$ is the element-wise positive part with $\max(0,P_{ij})$ in position $(i,j)$.
For $n \in \mathbb{N}$, $\Delta^{n}$ is the probability simplex $\{ \bm{p} \in \R_+^n \: \text{s.t.} \: \sum_i p_i = 1 \}$.
For $\bm{a} \in \Delta^{N_S}$ and $\bm{b} \in \Delta^{N_T}$, $\Pi(\bm{a}, \bm{b}) = \{ \Pb \in \R_+^{N_S \times N_T} \: \text{s.t.} \: \Pb \bm{1} = \bm{a} \: \text{and} \: \Pb^\top \bm{1} = \bm{b} \}$ is the transport polytope with marginals $(\bm{a}, \bm{b})$ while $\Pi(\bm{a}) = \{ \Pb \in \R_+^{N_S \times N_T} \: \text{s.t.} \: \Pb \bm{1} = \bm{a} \}$ is the semi-relaxed transport polytope.
For a set $\mathcal{E}$ and a divergence $D$, $\operatorname{Proj}^{D}_{\mathcal{E}}(\K)  =  \argmin_{\Pb \in\mathcal{E}} D(\Pb | \K)$.

\section{Proofs of the Optimal Transport Solutions}

\subsection{Optimal Transport with Global Constraint}\label{sec:proof_global}

\begin{prop}\label{prop:cot}
    Let $\psi : \R^{N_S} \to \R$ satisfy \cref{assumption_psi}. We define $\overline{\mathcal{B}}(\eta) \coloneqq \{ \Pb \: \text{s.t.} \: \sum_i \psi(\Pb_{i:}) \leq \eta \}$. Let $(\bm{a}, \bm{b}, \eta)$ be such that $\Pi(\bm{a}, \bm{b}) \cap \overline{\mathcal{B}}(\eta)$ has an interior point and let $\Pb^\star$ be a solution of
\begin{align}
    \min_{\Pb \in \Pi(\bm{a}, \bm{b})} \: \langle \Pb, \C \rangle \quad \text{s.t.} \quad  \Pb \in \overline{\mathcal{B}}(\eta) \:.
    \tag{ROT}
\end{align}
    Let $\varepsilon^\star$ be an optimal dual variable associated with $\Pb \in \overline{\mathcal{B}}(\eta)$.
    If $\varepsilon^\star > 0$, then for any $0 < \sigma \leq \varepsilon^\star$, it holds $\Pb^\star = \operatorname{Proj}^{D_\psi}_{\Pi(\bm{a}, \bm{b}) \cap \overline{\mathcal{B}}(\eta)}(\K_\sigma)$
    where $\K_\sigma \coloneqq \nabla \psi^*(-\C/ \sigma)$.
    One also has $\Pb^\star = \nabla \psi^*((\C - \bm{\lambda}^\star \oplus \bm{\mu}^\star) / \varepsilon^\star)$
    where $(\bm{\lambda}^\star, \bm{\mu}^\star, \varepsilon^\star)$ solve
    \begin{align}
        \max_{\bm{\lambda}, \bm{\mu}, \varepsilon>0}  \: \langle \bm{\lambda}, \bm{a} \rangle + \langle \bm{\mu}, \bm{b} \rangle + \varepsilon \Big(\sum_i \psi^*((\C_{i:} - \lambda_i \bm{1} - \bm{\mu}) / \varepsilon) - \eta \Big) \:.
    \tag{Dual-ROT}
    \end{align}
\end{prop}

\begin{proof}
    We first show that $\Pb^\star = \operatorname{Proj}^{D_\psi}_{\Pi(\bm{a}, \bm{b}) \cap \overline{\mathcal{B}}(\eta)}(\K_\varepsilon)$ before focusing on the dual problem.

    \textbf{\underline{Part I : Proof of the Bregman projection.}}

    Simplifying the constant terms $\operatorname{Proj}^{D_\psi}_{\Pi(\bm{a}, \bm{b}) \cap \overline{\mathcal{B}}(\eta)}(\K_\varepsilon)$ boils down to the following problem
    \begin{align}
        \min_{\Pb} \quad &\sum_i \psi(\Pb_{i:}) - \langle \Pb, \nabla \psi(\K_\sigma) \rangle \\
        \text{s.t.} \quad &\sum_i \psi(\Pb_{i:}) \leq \eta \\
        & \Pb \bm{1} = \bm{a}, \quad \Pb^\top \bm{1} = \bm{b} \\
        & \Pb \in \R_+^{n \times n} \:.
    \end{align}
    This problem is convex and strictly feasible. Strong duality holds thanks to Slater's constraint qualification. Therefore the KKT conditions \cite{boyd2004convex} are necessary and sufficient conditions for optimality.
    The Lagrangian can be expressed as
    \begin{align}
        \mathcal{L}(\Pb, \nu, \bm{\lambda}, \bm{\mu}) &= \sum_i \psi(\Pb_{i:}) - \langle \Pb, \nabla \psi(\K_\sigma) \rangle + \nu\Big(\sum_i \psi(\Pb_{i:}) - \eta \Big) \\
        &+ \langle \bm{\lambda}, \bm{a} - \Pb \bm{1} \rangle + \langle \bm{\mu}, \bm{b} - \Pb^\top \bm{1} \rangle - \langle \mathbf{\Omega}, \Pb \rangle \:.
    \end{align}
    Any optimal primal-dual variables $(\Pb^\star, \nu^\star, \bm{\lambda}^\star, \bm{\mu}^\star, \mathbf{\Omega}^\star)$ satisfies
    \begin{align}
            \nabla_{\Pb} \mathcal{L}(\Pb^\star,  \nu^\star, \bm{\lambda}^\star, \bm{\mu}^\star) &=  (\nu^\star + 1) \nabla \psi(\Pb^\star) - \nabla \psi(\K_\sigma) - \bm{\lambda}^\star \oplus \bm{\mu}^\star - \mathbf{\Omega}^\star = \bm{0} \:.
    \end{align}
    Note that by definition $\K_\sigma = \nabla \psi^\star(-\C/\epsilon)$ and thus $\nabla \psi(\K_\sigma) = \nabla \psi[\nabla \psi^\star(-\C/\sigma)] = -\C / \sigma$ \cite{rockafellar1997convex}. Thus we have 
    \begin{align}
        \C + \sigma (\nu^\star + 1) \nabla \psi(\Pb^\star) -  \sigma \bm{\lambda}^\star \oplus \bm{\mu}^\star - \sigma \mathbf{\Omega}^\star= \bm{0} \:.
    \end{align}
    Similarly, for the optimal transport problem \eqref{prop:cot}
    \begin{align}
        \min_{\Tb} \: \langle \Tb, \C \rangle \quad \text{s.t.} \quad  \Tb \in \Pi(\bm{a}, \bm{b}) \cap \overline{\mathcal{B}}(\eta) \:,
    \end{align}
    we get the optimality condition for optimal primal-dual variables $(\Tb^\star, \varepsilon^\star, \bm{\rho}^\star, \bm{\kappa}^\star, \mathbf{\Lambda}^\star)$
    \begin{align}
        \C + \varepsilon^\star \nabla \psi(\Tb^\star) -  \bm{\rho}^\star \oplus \bm{\kappa}^{\star} - \mathbf{\Lambda}^\star = \bm{0} \:.
    \end{align}
    Focusing on the original Bregman projection problem, we can then consider
    \begin{align}
        \left\{
        \begin{array}{ll}
            \bm{\lambda} &= \bm{\rho}^\star / \sigma \\
            \bm{\mu} &= \bm{\kappa}^\star / \sigma \\
            \nu &= \varepsilon^\star / \sigma - 1 \\
            \mathbf{\Omega} &= \mathbf{\Lambda}^\star / \sigma \\
            \Pb &= \Tb^\star \:.
        \end{array}
    \right.
    \end{align}
    With the above choice, 
    \begin{itemize}
        \item  $(\Pb, \nu, \bm{\lambda}, \bm{\mu}, \mathbf{\Omega})$ satisfies the first order optimality condition.
        \item $\Pb = \Tb^\star \in \Pi(\bm{a}, \bm{b}) \cap \mathcal{E}(\eta)$ thus the primal constraint is satisfied. 
        \item $\sigma \leq \varepsilon^\star$ implies that $\nu \geq 0$ and $\mathbf{\Omega}$ has positive entries thereby dual constraints are satisfied. 
        \item $\varepsilon^\star \neq 0$ thus by complementary slackeness $\psi(\Tb^\star) = \eta$ hence complementary slackness is also verified for $\Pb$ since $\Pb = \Tb^\star$.
    \end{itemize}
    Therefore the KKT conditions are met hence $(\Pb, \nu, \bm{\lambda}, \bm{\mu}, \mathbf{\Omega}) = (\Pb^\star, \nu^\star, \bm{\lambda}^\star, \bm{\mu}^\star, \mathbf{\Omega}^\star)$ and $\Pb^\star = \Tb^\star$.
\end{proof}

\textbf{\underline{Part II : Proof of dual ascent.}}

The optimal dual variables $(\bm{\lambda}^\star, \bm{\mu}^\star, \varepsilon^\star)$ solve the following problem 
\begin{align}
    &\max_{\bm{\lambda}, \bm{\mu}, \varepsilon > 0} \: \min_{\Pb \bm{\geq} \bm{0}} \: \langle \Pb, \C \rangle + \langle \bm{\lambda}, \bm{a} - \Pb \bm{1} \rangle + \langle \bm{\mu}, \bm{b} - \Pb^\top \bm{1} \rangle + \varepsilon \Big(\sum_i \psi(\Pb_{i:}) - \eta \Big) \\
    = &\max_{\bm{\lambda}, \bm{\mu}, \varepsilon> 0} \: \langle \bm{\lambda}, \bm{a} \rangle + \langle \bm{\mu}, \bm{b} \rangle - \varepsilon \eta + \min_{\Pb \bm{\geq} \bm{0}} \: \langle \Pb, \C - \bm{\lambda}\bm{1}^\top - \bm{1}\bm{\mu}^\top \rangle + \varepsilon \sum_i \psi(\Pb_{i:}) \\
    = &\max_{\bm{\lambda}, \bm{\mu}, \varepsilon>0} \: \langle \bm{\lambda}, \bm{a} \rangle + \langle \bm{\mu}, \bm{b} \rangle - \varepsilon \eta + \min_{\Pb \bm{\geq} \bm{0}} \: \sum_i \langle \Pb_{i:}, \C_{i:} - \lambda_i \bm{1} - \bm{\mu} \rangle + \varepsilon \psi(\Pb_{i:}) \\
    \stackrel{(\star)}{=} &\max_{\bm{\lambda}, \bm{\mu}, \varepsilon>0}  \: \langle \bm{\lambda}, \bm{a} \rangle + \langle \bm{\mu}, \bm{b} \rangle + \varepsilon \Big(\sum_i \psi^*((\C_{i:} - \lambda_i \bm{1} - \bm{\mu}) / \varepsilon) - \eta \Big) \:.
    \tag{Dual-ROT}
    \label{eq:dual-ROT}
\end{align}
In ($\star$) we have used that $\psi^*(\bm{x}) = \sup_{\bm{y} \bm{\geq} \bm{0}} \langle \bm{x}, \bm{y} \rangle - \psi(\bm{y})$.
From Danskin's theorem \cite{danskin1966theory}, one can recover the solution of the primal
\begin{align}
    \forall i, \: \Pb^\star_{i:} = \nabla \psi^*((\C_{i:} - \lambda^\star_i \bm{1} - \bm{\mu}^\star) / \varepsilon^\star) \:.
\end{align}
Using matrix notations yields $\Pb^\star = \nabla \psi^*((\C - \bm{\lambda}^\star \oplus \bm{\mu}^\star) / \varepsilon^\star)$
where $(\bm{\lambda}^\star, \bm{\mu}^\star, \varepsilon^\star)$ are the solution of the dual problem \eqref{eq:dual-ROT}.

\subsection{OT with Pointwise Constraints on Either Sources or Targets (OTARI-s and OTARI-t) : proof of \cref{prop:pcot}}\label{sec:proofs}

\begin{prop}
    Let $(\bm{a}, \bm{b}, \xi)$ be such that $\Pi(\bm{a}, \bm{b}) \cap \mathcal{B}_\psi(\xi)$ has an interior point and let $\Pb^\star$ solve
    \begin{align}
      \min_{\Pb \in \Pi(\bm{a}, \bm{b})} \: \langle \Pb, \C \rangle \quad \text{s.t.} \quad  \Pb \in \mathcal{B}_\psi(\xi) \:.
      \tag{OTARI-s}
  \end{align}
  Let $\bm{\varepsilon}^\star$ be the optimal dual variable associated with the constraint $\Pb \in \mathcal{B}_\psi(\xi)$.
  If $\bm{\varepsilon}^\star \bm{>} \bm{0}$, then it holds $\Pb^\star = \operatorname{Proj}^{D_\psi}_{\Pi(\bm{a}, \bm{b}) \cap \mathcal{B}_\psi(\xi)}(\K_\varepsilon)$ for any $0 < \varepsilon \leq \min_i{\varepsilon_i^\star}$. Moreover it holds $\Pb^\star = \nabla \psi^* \left(\diag(\bm{\varepsilon}^\star)^{-1} (\C - \bm{\lambda}^\star \oplus \bm{\mu}^\star) \right)$ where $(\bm{\lambda}^\star, \bm{\mu}^\star, \bm{\varepsilon}^\star)$ solve the following dual
  \begin{align}
    \max_{\bm{\lambda}, \bm{\mu}, \bm{\varepsilon} \bm{>} \bm{0}} \: \langle \bm{\lambda}, \bm{a} \rangle + \langle \bm{\mu}, \bm{b} \rangle + \left\langle \bm{\varepsilon}, \psi^*\left(\diag(\bm{\varepsilon})^{-1} (\C - \bm{\lambda} \oplus \bm{\mu}) \right) - \psi(\mathbf{e}_\xi) \bm{1}  \right\rangle \:.
    \tag{Dual-OTARI-s}
  \end{align}
\end{prop}

\begin{proof}
Again we break down the proof, focusing on the primal and then on the dual approach.

\textbf{\underline{Part I : Proof of the Bregman projection.}}

The proof is almost identical to the one in \cref{prop:cot}. We use the same notations for simplicity. The only difference brought by the pointwise constraint is that $\bm{\nu}^\star$ is now vectorial. The first-order optimality condition for the Bregman projection problem reads
\begin{align}
    \C + \sigma (\diag{(\bm{\nu}^\star)} + \mathbf{I}_{N_S}) \nabla \psi(\Pb^\star) -  \sigma \bm{\lambda}^\star \oplus \bm{\mu}^\star - \sigma \mathbf{\Omega}^\star= \bm{0} \:.
\end{align}
Again using the same notations as before, the first order KKT condition for problem \eqref{eq:cot} reads
\begin{align}
    \C + \diag(\bm{\varepsilon}^\star) \nabla \psi(\Tb^\star) -  \bm{\rho}^\star \oplus \bm{\kappa}^{\star} - \mathbf{\Lambda}^\star = \bm{0} \:.
\end{align}
We end the proof by following the same reasoning as for \cref{prop:cot}, choosing for any $i$, $\nu_i = \varepsilon_i^\star / \sigma - 1 \geq 0$.

\textbf{\underline{Part II : Dual Problem of \eqref{eq:pcOT}.}}

The optimization problem \eqref{eq:pcOT} writes
\begin{align}
    \min_{\Pb \in \Pi(\bm{a}, \bm{b})} \: \langle \Pb, \C \rangle \quad \text{s.t.} \quad  \Pb \in \mathcal{B}_\psi(\xi) \:.
\end{align}
Introducing the dual variables $\bm{\lambda} \in \mathbb{R}^n$ and $\bm{\mu} \in \mathbb{R}^n$ for the marginals and $\bm{\varepsilon} \in \mathbb{R}_+^{n}$ for the constraint $\Pb \in \mathcal{B}_\psi(\xi)$. The problem can be formulated as
\begin{align}
    \min_{\Pb \bm{\geq} \bm{0}} \: \max_{\bm{\lambda}, \bm{\mu}, \bm{\varepsilon} \bm{\geq} \bm{0}} \: \langle \Pb, \C \rangle + \langle \bm{\lambda}, \bm{a} - \Pb \bm{1} \rangle + \langle \bm{\mu}, \bm{b} - \Pb^\top \bm{1} \rangle + \langle \bm{\varepsilon}, \psi(\Pb) - \psi(\mathbf{e}_\xi) \bm{1} \rangle \:.
\end{align} 
When $\bm{\varepsilon}^\star \bm{>} \bm{0}$, relying on strong duality to invert the min and max operators, the problem reduces to
\begin{align}
    &\max_{\bm{\lambda}, \bm{\mu}, \bm{\varepsilon} \bm{>} \bm{0}} \: \min_{\Pb \bm{\geq} \bm{0}} \: \langle \Pb, \C \rangle + \langle \bm{\lambda}, \bm{a} - \Pb \bm{1} \rangle + \langle \bm{\mu}, \bm{b} - \Pb^\top \bm{1} \rangle + \langle \bm{\varepsilon}, \psi(\Pb) - \psi(\mathbf{e}_\xi) \bm{1} \rangle \\
    = &\max_{\bm{\lambda}, \bm{\mu}, \bm{\varepsilon} \bm{>} \bm{0}} \: \langle \bm{\lambda}, \bm{a} \rangle + \langle \bm{\mu}, \bm{b} \rangle - \langle \bm{\varepsilon}, \psi(\mathbf{e}_\xi) \bm{1} \rangle + \min_{\Pb \bm{\geq} \bm{0}} \: \langle \Pb, \C - \bm{\lambda}\bm{1}^\top - \bm{1}\bm{\mu}^\top \rangle + \langle \bm{\varepsilon}, \psi(\Pb) \rangle \\
    = &\max_{\bm{\lambda}, \bm{\mu}, \bm{\varepsilon} \bm{>} \bm{0}} \: \langle \bm{\lambda}, \bm{a} \rangle + \langle \bm{\mu}, \bm{b} \rangle - \langle \bm{\varepsilon}, \psi(\mathbf{e}_\xi) \bm{1} \rangle + \min_{\Pb \bm{\geq} \bm{0}} \: \sum_i \langle \Pb_{i:}, \C_{i:} - \lambda_i \bm{1} - \bm{\mu} \rangle + \varepsilon_i \psi(\Pb_{i:}) \\
    = &\max_{\bm{\lambda}, \bm{\mu}, \bm{\varepsilon} \bm{>} \bm{0}}  \: \langle \bm{\lambda}, \bm{a} \rangle + \langle \bm{\mu}, \bm{b} \rangle - \langle \bm{\varepsilon}, \psi(\mathbf{e}_\xi) \bm{1} \rangle + \sum_i \varepsilon_i \psi^*((\C_{i:} - \lambda_i \bm{1} - \bm{\mu}) / \varepsilon_i) \\
    \stackrel{(\star)}{=} &\max_{\bm{\lambda}, \bm{\mu}, \bm{\varepsilon} \bm{>} \bm{0}} \: \langle \bm{\lambda}, \bm{a} \rangle + \langle \bm{\mu}, \bm{b} \rangle + \left\langle \bm{\varepsilon}, \psi^*\left((\C - \bm{\lambda} \oplus \bm{\mu}) \oslash \bm{\varepsilon} \bm{1}^\top \right) - \psi(\mathbf{e}_\xi) \bm{1} \right\rangle \:.
    \label{eq:dual-otari-s}
    \tag{Dual-OTARI-s}
\end{align}
In ($\star$) we used that $\psi^*(\mathbf{X}) = \left(\psi^*(\mathbf{X}_{1:}), ..., \psi^*(\mathbf{X}_{N:})\right)^\top$.
From Danskin's theorem \cite{danskin1966theory}, one can recover the solution of the primal
\begin{align}
    \Pb^\star = \nabla \psi^* \left(\diag(\bm{\varepsilon}^\star)^{-1} (\C - \bm{\lambda}^\star \oplus \bm{\mu}^\star) \right)
\end{align}
where $(\bm{\lambda}^\star, \bm{\mu}^\star, \bm{\varepsilon}^\star)$ solve \eqref{eq:dual-otari-s}.
\end{proof}

\section{$\psi$-Bregman Projections for $\psi_{\KL}$ and $\psi_{2}$}\label{sec:proof_projs}

    In this section, we detail the expressions of the projections used in the alternating Bregman projection approach.

    \begin{figure*}[t]
        \begin{center}
        \centerline{\includegraphics[width=0.6\columnwidth]{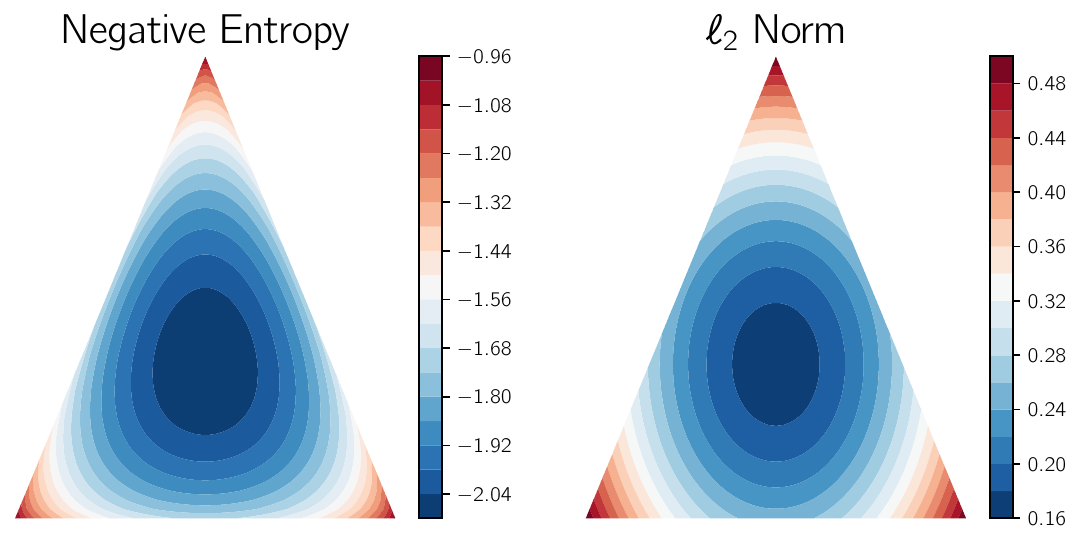}}
        \caption{$\sum_i \psi(p_i)$ plotted over the 3 dimensional probability simplex for $\psi_{\KL}$ (negative Shannon entropy) and $\psi_2 : \bm{x} \to \frac{1}{2} \| \bm{x} \|^2_2$. Unlike $\psi_{\KL}$, the level sets of $\psi_2$ intercept with the boundaries of the simplex thus leading to potentially sparse solutions when used to regularize OT.}
        \label{fig:Ps_vs_Pse}
        \end{center}
    \end{figure*}
    
    \subsection{KL Projections}
    
    \begin{prop}\label{prop:projections}
        When $D_\psi$ is the $\KL$ divergence $D_{\KL}$, one has for a matrix $\K \in \R_+^{N_S \times N_T}$,
          \begin{align}
            \operatorname{Proj}^{\KL}_{\Pi(\bm{a}) \cap \mathcal{B}_{\KL}(\xi)}(\mathbf{K}) = \diag(\mathbf{\Lambda} \mathbf{1})^{-1} \mathbf{\Lambda} \quad \text{with} \quad \mathbf{\Lambda} = \exp{(\diag(\bm{1} + \bm{\gamma}^\star)^{-1} \log \mathbf{K})}
        \end{align}
        where $\bm{\gamma}^\star \bm{\geq} \bm{0}$ is the optimal dual variable associated with the constraint $\mathcal{B}_{\KL}(\xi)$.
    \end{prop}
    
    \begin{proof}   
    The $\KL$ projection of a matrix $\K \in \R_+^{N_S \times N_T}$ onto $\Pi(\bm{a}) \cap \mathcal{B}_1(\xi)$ is the following problem.
    \begin{align}
        \min_{\Pb \in \R_+^{N_S \times N_T}} \quad &\KL(\Pb | \K) = \langle \Pb, \log (\Pb \oslash \K) - \bm{1}\bm{1}^\top \rangle \\
        \text{s.t.} \quad &\forall i \in \integ{N_S}, \: \operatorname{H}(\Pb_{i:}) \geq \eta \\
        & \Pb \bm{1} = \bm{a} \:.
    \end{align}
    where for $\bm{p} \in \R^{N_S}_+$, $\operatorname{H}(\bm{p}) = - \langle \bm{p}, \log \bm{p} - \bm{1} \rangle$ is the Shannon entropy.
    The associated Lagrangian writes
    \begin{align}
        \mathcal{L}(\Pb, \bm{\lambda}, \bm{\gamma}) &= \langle \Pb, \log \Pb - \log \K - \bm{1}\bm{1}^\top \rangle + \langle \bm{\gamma}, \eta \bm{1} - \operatorname{H}(\Pb) \rangle + \langle \bm{\lambda}, \bm{a} - \Pb \bm{1} \rangle \:.
    \end{align}
    
    Strong duality holds hence any optimal primal-dual variables $(\Pb^\star, \bm{\gamma}^\star, \bm{\lambda}^\star)$ must satisfy the KKT conditions. The first-order optimality condition gives
    \begin{align}
        \nabla_{\Pb} \mathcal{L} (\Pb^\star, \bm{\gamma}^\star, \bm{\lambda}^\star) &= \log \left( \Pb^\star \oslash \K \right) + \operatorname{diag}(\bm{\gamma}^\star)\log{\Pb^\star} - \bm{\lambda}^\star \bm{1}^\top = \bm{0} \:.
    \end{align}
    Isolating $\Pb^\star$ yields
    \begin{align}
        \forall (i,j) \in \integ{N_S} \times \integ{N_T}, \quad P^\star_{ij} = \frac{1}{u_i} \exp{((\log K_{ij})/(1 + \gamma^\star_i))}
    \end{align}
    where $u_i = \exp{(-\lambda_i/(1 + \gamma^\star_i))}$. Given the marginal constraint, we have 
    \begin{equation}
        u_i = a_i^{-1} \sum_{j \in \integ{N_T}} \exp{((\log K_{ij})/(1 + \gamma^\star_i))} \:.
    \end{equation}
    We are now left with $\Pb^\star$ as a function of $\bm{\gamma}$. Plugging $\Pb^\star$ in $\mathcal{L}$ yields the dual function $\bm{\gamma} \mapsto \mathcal{G}(\bm{\gamma})$. This function is concave (property of the dual problem) and its gradient reads:
    \begin{align}
        \nabla_{\bm{\gamma}} \mathcal{G}(\bm{\gamma}) = (\log \xi + 1 )\bm{1} - \operatorname{H}(\Pb^\star(\bm{\gamma})) \:.
    \end{align}
    Similarly to \cite{van2023snekhorn}, one can show that the above gradient is canceled for a unique $\overline{\bm{\gamma}}$. The optimal dual variable is then given by $\bm{\gamma}^\star = [\overline{\bm{\gamma}}]_+$.
            
        \end{proof}
    
    \subsection{Euclidean Projections}
    
    For the Euclidean case, we break down the projection into $\ell_2$ norm and marginal projections. Starting with the marginal projection, we have the following expression \cite{lorenz2021quadratically}    
\begin{equation}
    \operatorname{Proj}^{\ell_2}_{\Pi(\bm{a})}(\mathbf{K}) = [\bm{\lambda}^\star \bm{1}^\top + \mathbf{K}]_+
\end{equation}
    where $\bm{\lambda}^\star$ is such that $[\bm{\lambda}^\star \bm{1}^\top + \mathbf{K}]_+ \in \Pi(\bm{a})$.

    Focusing on the $\ell_2$ norm we have the following result.
    \begin{prop}
    One has
    \begin{equation}
        \operatorname{Proj}^{\ell_2}_{\mathcal{B}_2(\xi)}(\mathbf{K}) = \operatorname{diag}(\bm{\gamma}^\star)^{-1}\mathbf{K}
    \end{equation}
    where for any $i$, $\gamma^\star_i = \max \left(\xi^{1 / 2} \|\K_{i:}\|_2, 1 \right)$.
    \end{prop}

    \begin{proof}
    
    The $D_2$ projection of a matrix $\K \in \R_+^{N_S \times N_T}$ onto $\mathcal{B}_2(\xi)$ reduces to
    \begin{align}
        \min_{\Pb \in \R_+^{N_S \times N_T}} \quad &D_2(\Pb | \K_\varepsilon) = \frac{1}{2}\left\langle \mathbf{P}^{\odot 2}, \bm{1} \right\rangle - \left\langle \mathbf{P}, \mathbf{K} \right\rangle \\
        \text{s.t.} \quad
        &\forall i \in \integ{N_S}, \: \|\Pb_{i:} \|_2^2 \leq (1/\xi) \:.
    \end{align}
    Introducing the dual variable $\bm{\omega} \in \mathbb{R}_+^n$, the Lagrangian writes:
    \begin{align}
        \mathcal{L}(\Pb, \bm{\omega}, \bm{\Omega}) &=  \frac{1}{2}\left\langle \mathbf{P}^{\odot 2}, \bm{1} \right\rangle - \left\langle \mathbf{P}, \mathbf{K} \right\rangle + \frac{1}{2} \sum_i \omega_i \left(\|\Pb_{i:}\|_{2}^{2} - (1/\xi) \right) \:.
    \end{align}
    $\Pb^\star$ solves the primal problem if and only if there exists $\bm{\omega}^\star$ that satisfies the KKT conditions.
    The first-order condition yields
    \begin{align}
        \nabla_{\Pb} \mathcal{L} (\Pb^\star, \bm{\omega}^\star, \bm{\Omega}^\star) &= - \mathbf{K} + \operatorname{diag}(\bm{\omega}^\star + \bm{1})\Pb^\star = \bm{0} \:.
    \end{align}
    Hence it follows
    \begin{align}
        \Pb^\star = \operatorname{diag}(\bm{\omega}^\star + \bm{1})^{-1}\mathbf{K} \:.
    \end{align}

    To satisfy the KKT conditions, one has to find the root $\bm{\omega}^\star$ of the following independent problems 
    \begin{align}
        \forall i, \quad \left(\omega_i + 1 \right)^{2} =  \xi \|\K_{i:}\|_2^2 \:.
    \end{align}
    Thus we have $\omega_i + 1 = \xi^{1 / 2} \|\K_{i:}\|_2 $ and taking $\omega_i \geq 0$ into account yields the result.
    \end{proof}

\end{document}